\documentclass[twoside,leqno,twocolumn]{article}  
\usepackage{ltexpprt} 

\usepackage{amsmath}
\usepackage{bbm}
\usepackage[ruled]{algorithm}
\usepackage{tabularx}
\usepackage{url}
\usepackage{graphicx}
\usepackage[caption=false,font=footnotesize]{subfig}
\usepackage{multirow}
\usepackage{paralist}
\usepackage{booktabs}
\usepackage{epstopdf}
\usepackage{color}
\usepackage{stfloats}
\usepackage{booktabs}
\usepackage{tabularx}
\usepackage[noend]{algpseudocode}
\usepackage{mathtools}
\usepackage{url}
\usepackage{float}
\usepackage{amssymb}

\DeclareMathOperator{\trace}{trace}

\newcommand{\mat}[1]{\boldsymbol{#1}}
\renewcommand{\vec}[1]{\boldsymbol{\mathrm{#1}}}

\providecommand{\mD}{\ensuremath{\mat{D}}}

\providecommand{\mK}{\ensuremath{\mat{K}}}

\providecommand{\mU}{\ensuremath{\mat{U}}}

\providecommand{\mW}{\ensuremath{\mat{W}}}

\providecommand{\mY}{\ensuremath{\mat{Y}}}

\providecommand{\va}{\ensuremath{\vec{a}}}

\providecommand{\vd}{\ensuremath{\vec{d}}}
\providecommand{\ve}{\ensuremath{\vec{e}}}
\providecommand{\vf}{\ensuremath{\vec{f}}}
\providecommand{\vg}{\ensuremath{\vec{g}}}

\providecommand{\vl}{\ensuremath{\vec{l}}}

\providecommand{\vs}{\ensuremath{\vec{s}}}

\providecommand{\vu}{\ensuremath{\vec{u}}}
\providecommand{\vv}{\ensuremath{\vec{v}}}

\providecommand{\vx}{\ensuremath{\vec{x}}}
\providecommand{\vy}{\ensuremath{\vec{y}}}

\usepackage{xspace}

\newcommand{\sdp}{SDP\xspace}

\allowdisplaybreaks

\begin{document}

\title{\Large Fast Multiplier Methods to Optimize Non-exhaustive, Overlapping Clustering}
\author{Yangyang Hou\thanks{Department of Computer Science, Purdue University. Email:
$\left\{hou13, dgleich\right\}$@purdue.edu} \\
\and 
Joyce Jiyoung Whang\thanks{Department of Computer Engineering, Sungkyunkwan University. Email: jjwhang@skku.edu}  \\
\and 
David F. Gleich\footnotemark[1] \\
\and 
Inderjit S. Dhillon\thanks{Department of Computer Science, The University of Texas at
Austin. Email: inderjit@cs.utexas.edu}}
\date{}

\maketitle

%\pagenumbering{arabic}
%\setcounter{page}{1}%Leave this line commented out.

\begin{abstract} \small\baselineskip=9pt 
Clustering is one of the most fundamental and important tasks in data mining. Traditional clustering algorithms, such as K-means, assign every data point to exactly one cluster. However, in real-world datasets, the clusters may overlap with each other. Furthermore, often, there are outliers that should not belong to any cluster. We recently proposed the NEO-K-Means (Non-Exhaustive, Overlapping K-Means) objective as a way to address both issues in an integrated fashion. Optimizing this discrete objective is NP-hard, and even though there is a convex relaxation of the objective, straightforward convex optimization approaches are too expensive for large datasets. A practical alternative is to use a low-rank factorization of the solution matrix in the convex formulation. The resulting optimization problem is non-convex, and we can locally optimize the objective function using an augmented Lagrangian method. In this paper, we consider two fast multiplier methods to accelerate the convergence of an augmented Lagrangian scheme: a proximal method of multipliers and an alternating direction method of multipliers (ADMM). For the proximal augmented Lagrangian or proximal method of multipliers, we show a convergence result for the non-convex case with bound-constrained subproblems. These methods are up to 13 times faster---with no change in quality---compared with a standard augmented Lagrangian method on problems with over 10,000 variables and bring runtimes down from over an hour to around 5 minutes. 

\end{abstract}

\section{Introduction}
Traditional clustering algorithms, such as $k$-means, produce a disjoint, exhaustive clustering, i.e., the clusters are pairwise disjoint, and every data point is assigned to some cluster. However, in real-world datasets, the clusters may overlap with each other, and there are often outliers that should not belong to any cluster. We recently proposed the NEO-K-Means (Non-Exhaustive, Overlapping K-Means) objective as a generalization of the $k$-means clustering objective that allows us to simultaneously identify overlapping clusters as well as outliers~\cite{whang-2015-neo}. Hence, it  produces a non-exhaustive clustering. Curiously, both operations appear to be necessary because the outliers induce non-obvious effects when the clusters are allowed to overlap. It has been shown that the NEO-K-Means objective is effective in finding ground-truth clusters in data clustering problems. Furthermore, by considering a weighted and kernelized version of the NEO-K-Means objective, we can also tackle the problem of finding overlapping communities in social and information networks.

%at SDM2015 
%The NEO-K-Means idea can be used to identify coherent groups in data clustering problems that accurately match ground truth data as well as identify overlapping communities via a relationship with graph clustering problems and a weighted, kernelized method. 
%We review the objective in Section~\ref{sec:pre}.

There are currently two practical methods to optimize the non-convex NEO-K-Means objective for large problems: the iterative NEO-K-Means algorithm~\cite{whang-2015-neo} that generalizes Lloyd's algorithm~\cite{lloyd} and an augmented Lagrangian algorithm to optimize a non-convex, low-rank semidefinite programming (SDP) relaxation of the NEO-K-Means objective~\cite{Hou-Whang-2015-lrsdp-neo}. The iterative algorithm is fast, but it tends to get stuck into regions where the more sophisticated optimization methods can make further progress. The augmented Lagrangian method for the non-convex objective, when started from the output of the iterative algorithm, is able to make further progress on optimizing the objective function. In addition, the augmented Lagrangian method tends to achieve better $F_1$ performance on identifying ground-truth clusters and produce better overlapping communities in real-world networks than the simple iterative algorithm~\cite{Hou-Whang-2015-lrsdp-neo}. In this paper, our goal is to improve upon the augmented Lagrangian method to optimize the low-rank SDP for the NEO-K-Means objective more quickly.

%an iterative method that generalizes Lloyd's method for the $k$-means problem
%The iterative algorithm is fast, although it frequently falls into regions where it no longer makes progress but that do not appear to be local optima.
%the iterative method often proceeds to a place where the more sophisticated optimization methods can make further progress

The optimization problem that results from the low-rank strategy on the convex SDP is a non-convex, quadratically constrained, bound-constrained problem. We consider two \emph{multiplier methods} for this problem. The first method adds a proximal regularizer to the augmented Lagrangian method. This general strategy is called either the proximal augmented Lagrangian method~(e.g.,~\cite{Humes-2004-prox-aug-lagrangian}) or the proximal method of multipliers~\cite{Rockafellar-1976-AugLagProx}. The second method is an alternating direction method of multipliers (ADMM) strategy for our objective function. Both strategies, when specialized on the NEO-K-Means problem, have the potential to accelerate our solution process. 
%as discussed more in Section~\ref{sec:methods}. 

There is an extensive literature on both strategies for convex optimization~\cite{Boyd-2011-admm,Friedlander-2012-primal-dual,Rockafellar-1976-AugLagProx} and there are a variety of convergence theories in the non-convex case~\cite{Magnusson-2014-nonconvex-admm,Pennanen-2002-proxMOR,Iusem-2003-inexact-prox}. However, we were unable to identify any existing convergence guarantees for these methods that mapped to our specific instantiations with bound-constrained subproblems. Towards that end, we specialize a general convergence result about the proximal augmented Lagrangian or proximal method of multipliers due to Pennanen~\cite{Pennanen-2002-proxMOR} to our algorithm. The resulting theorem is a general convergence result about the proximal augmented Lagrangian method for non-convex problems with bound-constrained subproblems (Theorem~\ref{theorem:1}). The proof involves adapting a few details from Pennanen to our case.

We evaluate the resulting methods on real-world problems where the existing augmented Lagrangian takes over an hour of computation time. The proximal augmented Lagrangian strategy tends to run about $3-6$ times faster, and the ADMM strategy tends to run about $4-13$ times faster bringing the runtimes of these methods down into range of $5$ to $10$ minutes. The iterative method, in contrast, runs in seconds -- so there is still a considerable gap between the approaches. That said, the optimization based approaches have runtimes that are reasonable for a pipeline-style analysis and cases where the data collection itself is highly time-consuming as would be common in many datasets from the biological and physical sciences. 

In summary:
\begin{compactitem}
\item We propose two algorithms to optimize the non-convex problem for non-exhaustive, overlapping clustering: a proximal augmented Lagrangian method and an ADMM method.
\item We specialize a general convergence result about the proximal method of multipliers for non-convex problems to the bound-constrained proximal augmented Lagrangian method to have a sound convergence theory. 
\item We show that these new methods reduce the runtime for problems where the classical augmented Lagrangian method takes over an hour to the range of $5$ to $10$ minutes with no change in quality.
\end{compactitem}

The rest of the paper is organized as follows. In Section~\ref{sec:pre}, we review the NEO-K-Means objective and its low-rank SDP formulation, and in Section~\ref{sec:lrsdp}, we formally describe the classical augmented Lagrangian method. In Section~\ref{sec:methods}, we present our two multiplier methods: the proximal augmented Lagrangian method, and an ADMM for the NEO-K-Means low-rank SDP. For the proximal augmented Lagrangian method, we present the convergence analysis in Section~\ref{sec:conv}. In Section~\ref{sec:sadmm}, we discuss simplified ADMM variants. Finally, we present experimental results in Section~\ref{sec:results}, and discuss future work in Section~\ref{sec:discuss}.

%which aim to accelerate ADMM don't improve the performance

\section{The NEO-K-Means Objective}
\label{sec:pre}
The goal of non-exhaustive, overlapping clustering is to find a set of cohesive clusters such that clusters are allowed to overlap with each other and outliers are not assigned to any cluster. That is, given a set of data points $\mathcal{X} = \{ {\bf x}_1, {\bf x}_2, ..., {\bf x}_n \}$, we find a set of clusters $\mathcal{C}_1, \mathcal{C}_2, ..., \mathcal{C}_k$ such that $\mathcal{C}_1 \cup \mathcal{C}_2 \cup ... \cup \mathcal{C}_k \subseteq \mathcal{X}$ and $\mathcal{C}_i \cap \mathcal{C}_j \neq \emptyset$ for some $i \neq j$. 

To find such clusters, we proposed the NEO-K-Means objective function in~\cite{whang-2015-neo}. The NEO-K-Means objective is an intuitive variation of the classical $k$-means where two parameters $\alpha$ and $\beta$ are introduced to control the amount of overlap and non-exhaustiveness, respectively. We also found that optimizing a weighted and kernelized NEO-K-Means objective is equivalent to optimizing normalized cuts for overlapping community detection~\cite{whang-2015-neo}. 

Let us define an assignment matrix $\mU=[u_{ij}]_{n \times k}$ such that $u_{ij}=1$ if a data point $\vx_i$ belongs to $\mathcal{C}_j$; and $u_{ij}=0$ otherwise. Let $\mathbbm{I}\{ exp \}$ denote the indicator function such that $\mathbbm{I}\{ exp \} = 1$ if $exp$ is true; 0 otherwise. Given a positive weight for each data point $w_i$, and a nonlinear mapping $\phi$, the weighted kernel NEO-K-Means objective function is defined as follows:
\begin{equation} \label{eq:neo}
\begin{array}{ll}
  \underset{\mU}{\text{minimize}} & \sum_{c=1}^k \sum_{i=1}^n u_{ic} w_i \| \phi(\vx_i) - \mathbf{m}_c \|^2 \\
  & \text{   where } \mathbf{m}_c = \frac{\sum_{i=1}^n u_{ic} w_i \phi(\vx_i)}{\sum_{i=1}^n u_{ic} w_i}\\ \addlinespace
  \text{subject to} & \trace(\mU^T \mU) = (1+\alpha) n, \\
  & \textstyle\sum_{i=1}^n \mathbbm{I}\{ (\mU\mathbf{1})_i = 0 \} \leq \beta n.
\end{array}
\end{equation}
%This objective function implies that we make $(1+\alpha)n$ assignments such that the sum of the squared distances between a data point and its cluster center is minimized.
This objective function implies that $(1+\alpha)n$ assignments are made while minimizing the sum of the squared distances between a data point and its cluster center. Also notice that at most $\beta n$ data points are allowed to have no membership in any cluster. If $\alpha = 0$ and $\beta = 0$, then this objective is equivalent to the classical weighted kernel $k$-means objective. Some guidelines about how to select $\alpha$ and $\beta$ have been described in~\cite{whang-2015-neo}. To optimize the objective function (\ref{eq:neo}), a simple iterative algorithm has also been proposed in \cite{whang-2015-neo}. However, the simple iterative algorithm tends to get stuck at a local optimum that can be far away from the global optimum, like the standard $k$-means algorithm~\cite{lloyd}.

The following optimization problem is a non-convex relaxation of the NEO-K-Means problem that was developed in our previous work~\cite{Hou-Whang-2015-lrsdp-neo}. We call it the low-rank \sdp based on its derivation as a low-rank heuristic for solving large-scale SDPs. We introduce a bit of notation to state the problem. Let $\mK$ be a standard kernel matrix ($K_{ij} = \phi(\vx_i)^T \phi(\vx_j)$), let $\mW$ denote a diagonal weight matrix such that $W_{ii}=w_i$ indicates the weight of data point $i$, and let $\vd$ denote a vector of length $n$ where $d_i = w_iK_{ii}$. In terms of the solution variables, let $\vf$ be a length $n$ vector where $f_i$ is a real-valued count of the number of clusters data point $i$ is assigned to, and let $\vg$ be a length $n$ vector where $g_i$ is close to 0 if $i$ should not be assigned to any cluster and $g_i$ is close to 1 if $i$ should be assigned to a cluster.  The solution matrix $\mY$ represents a relaxed, normalized assignment matrix where $Y_{ij}$ indicates that data point $i$ should be in cluster $j$ with columns normalized by the cluster size. The low-rank \sdp optimization problem for (\ref{eq:neo}) is then
\begin{equation} \label{eq:neo-lr}
\begin{array}{lll}
\underset{\mY, \vf, \vg, \vs, r}{\text{minimize}} & \vf^T \vd  - \trace(\mY^T \mK \mY)  \\[1ex]
\text{subject to} 
  & k = \trace(\mY^T \mW^{-1} \mY) & (a) \\
	& 0 = \mY \mY^T \ve - \mW \vf & (b) \\
	& 0= \ve^T \vf - (1+\alpha) n & (c)\\
	& 0 = \vf - \vg - \vs & (d) \\
	& 0 = \ve^T \vg - (1-\beta) n - r & (e) \\
	& Y_{i,j} \ge 0, \vs \ge 0, r \ge 0 \\
	&  0 \le \vf \le k \ve, 0 \le \vg \le 1
\end{array}
\end{equation}
where $\vs$ and $r$ are slack variables to convert the inequality constraints into equality constraints. The objective function is derived following a standard kernelized conversion. Constraint (a) gives the normalization condition on the variable $\mY$ to normalize for cluster-size; constraint (b) requires that the number of assignments listed in $\vf$ corresponds to the number in the solution matrix $\mY$; constraint (c) bounds the total number of assignments as $(1+\alpha) n$; constraint (d) is equivalent to $\vf \ge \vg$; and constraint (e) enforces the number of assigned data points to be at least $(1-\beta)n$; the remaining bound constraints enforce simple non-negativity and upper-bounds on the number of cluster assignments. 
We will discuss how to solve the low-rank \sdp problem in the next section.

\section{The Augmented Lagrangian Method for the NEO-K-Means Low-Rank SDP}
\label{sec:lrsdp}
To solve the low-rank \sdp problem~\eqref{eq:neo-lr}, the classical augmented Lagrangian method (ALM) has been used in~\cite{Hou-Whang-2015-lrsdp-neo}. The augmented Lagrangian technique is an iterative process where each iteration is done by minimizing an augmented Lagrangian problem that includes a current estimate of the Lagrange multipliers for the constraints as well as a quadratic penalty term that enforces the feasibility of the solution. We introduce it here because we will draw heavily on the notation for our subsequent results.

%Classical augmented Lagrangian framework is shown successful to solve the low-rank \sdp problem~\eqref{eq:neo-lr} \cite{Hou-Whang-2015-lrsdp-neo}. The iterations involve minimizing an augmented Lagrangian subproblem at each step that includes a current estimate of the Lagrange multipliers for the constraints as well as a quadratic penalty term to ensure the nonlinear constraints qualification.

%Let $\vec{\lambda} = [\lambda_1; \lambda_2; \lambda_3]$ be the Lagrange multipliers associated with the three scalar constraints $(a), (c), (d)$, and $\vec{\mu}$ and $\vec{\gamma}$ be the Lagrange multipliers associated with the vector constraints $(b)$ and $(e)$, respectively. Let $\sigma > 0$ be a penalty parameter. The augmented Lagrangian for~\eqref{eq:neo-lr} is:

Let $\vec{\lambda} = [\lambda_1; \lambda_2; \lambda_3]$ denote the Lagrange multipliers for the three scalar constraints $(a), (c), (e)$. For the vector constraints $(b)$ and $(d)$, let $\vec{\mu}$ and $\vec{\gamma}$ denote the corresponding Lagrange multipliers, respectively. Let $\sigma$ be a positive penalty parameter. Then, the augmented Lagrangian for~\eqref{eq:neo-lr} is:
%\begin{equation} \label{eq:al}
\begin{align*}
 & \mathcal{L}_A (Y, \vf, \vg, \vs, r;  \vec{\lambda}, \vec{\mu}, \vec{\gamma}, \sigma) = \\
 & \qquad \underbrace{\vf^T\vd - \trace(\mY^T\mK\mY)}_{\text{the objective}} \\
 & \qquad - \lambda_1 (\trace(\mY^T \mW^{-1} \mY)-k) \\
 & \qquad \qquad + \frac{\sigma}{2} (\trace(\mY^T \mW^{-1} \mY)-k)^2\\
 & \qquad -\vec{\mu}^T( \mY\mY^T \ve - \mW \vf) \\
 & \qquad \qquad + \frac{\sigma}{2} ( \mY\mY^T \ve - \mW \vf)^T ( \mY\mY^T \ve - \mW \vf) \\
 & \qquad -\lambda_2 (\ve^T\vf- (1+\alpha)n) \\
 & \qquad \qquad + \frac{\sigma}{2} (\ve^T\vf- (1+\alpha)n)^2 \\
 & \qquad -\vec{\gamma}^T (\vf- \vg - \vs) \\
 & \qquad \qquad +  \frac{\sigma}{2} (\vf-\vg - \vs)^T(\vf-\vg -\vs) \\
 & \qquad -\lambda_3 (\ve^T\vg -(1-\beta)n - r) \\
 & \qquad \qquad + \frac{\sigma}{2} (\ve^T\vg -(1-\beta)n - r)^2
\end{align*}
%\end{equation} 

%\begin{equation} \label{eq:al}
%\begin{aligned}
% & \mathcal{L}_{\mathcal{A}} (Y, \vf, \vg, \vs, r;  \vec{\lambda}, \vec{\mu}, \vec{\gamma}, \sigma) = \underbrace{\vf^T\vd - \trace(\mY^T\mK\mY)}_{\text{the objective}} \\
% & \qquad - \lambda_1 (\trace(\mY^T \mW^{-1} \mY)-k)  + \frac{\sigma}{2} (\trace(\mY^T \mW^{-1} \mY)-k)^2\\
% & \qquad -\vec{\mu}^T( \mY\mY^T \ve - \mW \vf)  + \frac{\sigma}{2} ( \mY\mY^T \ve - \mW \vf)^T ( \mY\mY^T \ve - \mW \vf) \\
% & \qquad -\lambda_2 (\ve^T\vf- (1+\alpha)n) + \frac{\sigma}{2} (\ve^T\vf- (1+\alpha)n)^2 \\
% & \qquad -\vec{\gamma}^T (\vf- \vg - \vs) +  \frac{\sigma}{2} (\vf-\vg - \vs)^T(\vf-\vg -\vs) \\
% & \qquad -\lambda_3 (\ve^T\vg -(1-\beta)n - r) + \frac{\sigma}{2} (\ve^T\vg -(1-\beta)n - r)^2
%\end{aligned}
%\end{equation} 

At each iteration of the augmented Lagrangian framework, the following subproblem is solved:
\begin{equation}
\begin{aligned}
& {\text{minimize}}
& & \mathcal{L}_{\mathcal{A}} (\mY, \vf, \vg, \vs, r;  \vec{\lambda}, \vec{\mu}, \vec{\gamma}, \sigma)   \\
& \text{subject to}
& & Y_{i,j} \ge 0, \vs \geq 0, r \geq 0, \\
&&& 0 \leq  \vf \leq k\ve, 0 \leq \vg \leq 1.
\end{aligned}
\end{equation}
To minimize the subproblem with respect to the variables $\mY$, $\vf$, $\vg$, $\vs$, and $r$, we can use a limited-memory BFGS with bound constraints algorithm \cite{Byrd-1995-lbfgsb}. In~\cite{Hou-Whang-2015-lrsdp-neo}, it has been shown that this technique produces reasonable solutions for the NEO-K-Means objective. In particular, when the clustering performance is evaluated on real-world datasets, this technique has been shown to be effective in finding the ground-truth clusters. Furthermore, by optimizing the weighted kernel NEO-K-Means, this technique is also able to find cohesive overlapping communities in real-world networks. The empirical success of the augmented Lagrangian framework motivates us to investigate developing faster solvers for the NEO-K-Means low-rank SDP problem, which will be discussed in the next section.

%is used to minimize the subproblem with respect to all the variables. Experiments from \cite{Hou-Whang-2015-lrsdp-neo} show that this framework can achieve the best $F1$ performance with respect to ground-truth clusters on real-world vector clustering problem and produces the best quality communities among all clustering algorithms on real-world networks. 
%Inspired by the success of augmented Lagrangian framework, we are interested in more efficient variants of classical augmented Lagrangian method.

\section{Fast Multiplier Methods for the NEO-K-Means Low-Rank SDP}
\label{sec:methods}
There is a resurgence of interest in proximal point methods and alternating methods for convex and nearly convex objectives in machine learning due to their fast convergence rate. Here we propose two variants of the classical augmented Lagrangian approach on problem~\eqref{eq:neo-lr} that can utilize some of these techniques for improved speed.  

\subsection{Proximal Augmented Lagrangian (PALM).}
\label{sec:palm}
%By applying the proximal point method to both the primal and dual problem,  
The proximal augmented Lagrangian method differs from the classical augmented Lagrangian method only in an additional proximal regularization term for primal updates. This can be considered as a type of simultaneous primal-dual proximal-point step that helps to regularize the subproblems solved at each step. This idea leads to the following iterates:
%%\begin{equation}
%\begin{align*}
%\vx^{k+1} & = \underset{\vx \in \mathcal{C} } {\mathrm{arg min}}  & \mathcal{L}_{\mathcal{A}} (\vx;  \vec{\lambda}^k, \vec{\mu}^k, \vec{\gamma}^k, \sigma) \\
%& & \quad + \frac{1}{2\tau} \| \vx - \vx^k\|^2
%\end{align*}
%%\end{equation}
\begin{displaymath}
\vx^{k+1}  = \underset{\vx \in \mathcal{C} } {\mathrm{arg min}}  \mathcal{L}_{\mathcal{A}} (\vx;  \vec{\lambda}^k, \vec{\mu}^k, \vec{\gamma}^k, \sigma)  + \frac{1}{2\tau} \| \vx - \vx^k\|^2
\end{displaymath}
where $\vx$ represents $[\vy;\vf;\vg;\vs;r]$ for simplicity with $\vy = \mY(:)$ vectorized by column. Then we update the multipliers $\vec{\lambda}, \vec{\mu}, \vec{\gamma}$ as in the classical augmented Lagrangian. We may also need to update the penalty parameter $\sigma$ and the proximal parameter $\tau$ respectively.

We use a limited-memory BFGS with bound constraints to solve the new subproblem with respect to the variable $\vx$. If we let $\tau = \sigma$, this special case is called proximal method of multipliers, first introduced in \cite{Rockafellar-1976-AugLagProx}. The proximal method of multipliers has better theoretical convergence guarantees for convex optimization problems (compared with the augmented Lagrangian)~\cite{Rockafellar-1976-AugLagProx}. In this non-convex setting, we believe it is likely to help to improve conditioning of the Hessian's in the subproblems and thus reduce the solution time for each subproblem. And this is indeed what we find.

\subsection{Alternating Direction Method of Multipliers (ADMM).}
\label{sec:admm}

There are four sets of variables in problem~\eqref{eq:neo-lr} ($\mY$, $\vf$, $\vg$ and slack variables). We can use this structure to break the augmented Lagrangian function into smaller subproblems for each set of variables. Some of these subproblems are then easier to solve. For example, updating variable $\vf$ alone is a simple convex problem, thus it is very efficient to have a globally optimal solution. The alternating direction method of multipliers approach of updating block variables $\mY$, $\vf$, $\vg$, $\vs$ and $r$ respectively, utilizes this property, which leads to the following iterates:%\begin{equation}
\begin{align*}
\mY^{k+1} & = \underset{\mY} {\mathrm{arg min}} \mathcal{L}_{\mathcal{A}} (\mY, \vf^k,\vg^k, \vs^k, r^k;\\
& \vec{\lambda}^k, \vec{\mu}^k, \vec{\gamma}^k, \sigma) \\
\vf^{k+1} & = \underset{\vf} {\mathrm{arg min}} \mathcal{L}_{\mathcal{A}} (\mY^{k+1}, \vf,\vg^k, \vs^k, r^k; \\
& \vec{\lambda}^k, \vec{\mu}^k, \vec{\gamma}^k, \sigma) \\
\vg^{k+1} & = \underset{\vg} {\mathrm{arg min}} \mathcal{L}_{\mathcal{A}} (\mY^{k+1}, \vf^{k+1},\vg, \vs^k, r^k;\\
& \vec{\lambda}^k, \vec{\mu}^k, \vec{\gamma}^k, \sigma) \\
\vs^{k+1} & = \underset{\vs} {\mathrm{arg min}} \mathcal{L}_{\mathcal{A}} (\mY^{k+1}, \vf^{k+1},\vg^{k+1}, \vs, r^k; \\
& \vec{\lambda}^k, \vec{\mu}^k, \vec{\gamma}^k, \sigma) \\
r^{k+1} & = \underset{r} {\mathrm{arg min}} \mathcal{L}_{\mathcal{A}} (\mY^{k+1}, \vf^{k+1},\vg^{k+1}, \vs^{k+1}, r; \\
& \vec{\lambda}^k, \vec{\mu}^k, \vec{\gamma}^k, \sigma)
\end{align*}%\end{equation}
then the multipliers $\vec{\lambda}$, $\vec{\mu}$, $\vec{\gamma}$ and the penalty parameter $\sigma$ are updated accordingly. 

We expect that this strategy will aid convergence because it decouples the update of $\mY$ from the update of $\vf$. In the problem with all variables, the interaction of these terms has the strongest non-convex interaction.  We now detail how we solve each of the subproblems.

\textbf{Update $\mY$}. We use a limited-memory BFGS with bound constraints to solve the subproblem with respect to the variables $\mY$ since it is non-convex.  

\textbf{Update $\vf$ and $\vg$}. The update for $\vf$ and $\vg$ respectively both have the following general form:
\begin{equation} \label{eq:quadratic}
\begin{aligned}
& \underset{\vx}{\text{minimize}}
& & f(\vx) = \vx^T \va  + \frac{\sigma}{2} \vx^T \mD \vx + \frac{\sigma}{2} (\ve^T\vx)^2\\
& \text{subject to}
& & 0 \leq \vx \leq b \ve\\
\end{aligned}
\end{equation}
where $\ve$ is the vector of all 1s and $\mD$ is a positive diagonal matrix. To solve this, we use ideas similar to~\cite[S6.2.5]{Parikh-2014-prox}. 
%We have the gradient $\nabla f(\vx) = \va + \sigma \mD \vx + \sigma \ve\ve^T \vx$. 
Let $\tau = \ve^T \vx$, thus $0 \leq \tau \leq bn$. We solve this problem by finding roots of the following function $F(\tau)$: 
\begin{displaymath}
F(\tau) = \tau - \ve^T P[-\tfrac{1}{\sigma} \mD^{-1} (\va + \sigma\tau \ve); 0, b]
\end{displaymath}
where the function $P[\vx; 0, b]$ projects the point $\vx$ onto the rectangular box $[0, b]$. (To find these roots, bisection suffices because $F(0) \leq 0$ and $F(bn) \geq 0$.) This is a globally optimal solution by the following lemma.

\begin{lemma}
$\vx^* = P[-\frac{1}{\sigma} \mD^{-1} (\va + \sigma \tau^* \ve);0, b]$, where $\tau^*$ is the root of $F(\tau)$, satisfies the first order KKT conditions: $\vx^* - P[\vx^* - \nabla f(\vx^*); 0, b] = 0$ (this form is given in equation 17.51 of~\cite{Nocedal-2006-optimization}).
\end{lemma}

\begin{proof}
We have three cases: $x^*_i = 0$; $x^*_i = b$; and $0 < x^*_i < b$ for any $i$. 

For $x^*_i = 0$, which means $a_i + \sigma\tau \ge 0$, we have 
\begin{equation*}
\begin{aligned}
& x^*_i  - P[x^*_i - (a_i + \sigma d_i x^*_i + \sigma \tau); 0, b]  \\
= \quad & - P[-a_i - \sigma \tau; 0, b] = 0.
\end{aligned}
\end{equation*}

For $x^*_i = b$, which means $-(a_i + \sigma\tau)/ (\sigma d_i)  \geq b$, we have 
\begin{equation*}
\begin{aligned}
& x^*_i  - P[x^*_i - (a_i + \sigma d_i x^*_i + \sigma \tau); 0, b] \\
= \quad & b  - P[ b -(a_i + \sigma d_i b + \sigma \tau); 0, b] = b - b = 0.
\end{aligned}
\end{equation*}

For $0 < x^*_i < b$, which means $x^*_i  = -(a_i + \sigma\tau)/ (\sigma d_i) $, we have 
\begin{equation*}
\begin{aligned}
& x^*_i  - P[x^*_i - (a_i + \sigma d_i x^*_i + \sigma \tau); 0, b] \\
= \quad& x^*_i - P[x^*_i; 0, b] =x^*_i - x^*_i =  0. \qquad \qquad \blacksquare
\end{aligned}
\end{equation*}
%Thus, $\vx^*$ is the globally optimal solution. 
\end{proof}

\textbf{Update $\vs$ and $r$}. These updates just require solving one variable quadratic optimization with simple bound constraints; the result is a simple update procedure.

\section{Convergence Analysis of the Proximal Augmented Lagrangian}
\label{sec:conv}
We use both the proximal augmented Lagrangian and the ADMM strategy on the problem without any convexity. For these cases, local convergence is the best we can achieve. We now establish a general convergence result for the proximal augmented Lagrangian with bound constraints. We observed empirical convergence of the ADMM method, but currently lack any theoretical guarantees.
 
From Pennanen~\cite{Pennanen-2002-proxMOR}, we know that the proximal method of multipliers is locally convergent for a general class of problems with sufficient assumptions.  We will show that our proximal method of multipliers algorithm applied to~\eqref{eq:neo-lr} can be handled by their approach and we extend their analysis to our case. Because we are imitating the analysis from Pennanen for a specific new problem, we decided to explicitly mimic the original language to highlight the changes in the derivation. Thus, there is a high degree of textual overlap between the following results and~\cite{Pennanen-2002-proxMOR}.

First, we state some notation and one useful fact. The indication function $\delta_\mathcal{C}$ of a set $\mathcal{C}$ in Hilbert Space $\mathcal{H}$ has value 0 for $x \in \mathcal{C}$ and $+\infty$ otherwise. The subdifferential of $\delta_\mathcal{C}$ is the normal cone operator of $\mathcal{C}$: $N_{\mathcal{C}} (\vx) = \{ \vv \in \mathcal{H} |  \langle \vv, \vy - \vx \rangle \leq 0, \forall \vy \in \mathcal{C} \}.$

\begin{proposition}
\label{prop}
 Let $\bar{\vx}$ be a solution to problem of minimizing $f(\vx)$ on $\mathcal{C}$ and suppose $f$ is differentiable at $\bar{\vx}$,  then 
\[
\nabla f(\bar{\vx}) + N_{\mathcal{C}} (\bar{\vx}) \ni  0.
\]
\end{proposition}

\begin{proof}
We need to show that $\nabla f(\bar{\vx}) + N_{\mathcal{C}} (\bar{\vx}) \ni  0$  is equivalent to $
\nabla f(\bar{\vx})^T (\vy - \bar{\vx}) \geq 0 \quad \text{for all } \vy \in \mathcal{C}$, which is clear from the definition of the normal cone. $\blacksquare$ \qquad \end{proof} 

To simplify the convergence behavior analysis of the  proximal method of multipliers on~\eqref{eq:neo-lr}, we generalize the optimization problem in the following form:
\begin{equation} \label{eq:general-form}
\begin{aligned}
& \underset{\vx}{\text{minimize}}
& & f(\vx) \\
& \text{subject to}
& & c(\vx) = \vec{0}, \\
& & &\vl \leq \vx \leq \vu 
\end{aligned}
\end{equation}
where $f(\vx)$ and $c(\vx)$ are continuous and differentiable.  Let $\mathcal{C}$ be the closed convex sets corresponding to simple bound constrains $\vl \leq \vx \leq \vu$. 

The Lagrangian and augmented Lagrangian function are defined respectively as follows:
\begin{displaymath}
\mathcal{L} (\vx, \vec{\lambda}) = f(\vx) + \vec{\lambda}^Tc(\vx)
\end{displaymath}
\begin{displaymath}
\mathcal{L}_{\mathcal{A}} (\vx, \vec{\lambda}, \sigma) = f(\vx) + \vec{\lambda}^Tc(\vx) + \frac{\sigma}{2} \| c(\vx) \|^2.
\end{displaymath}
The multipliers $\vec{\lambda}$ can be added or subtracted. We choose adding the multipliers here in order to be consistent with the analysis in \cite{Pennanen-2002-proxMOR}.

A point $(\bar{\vx}, \bar{\vec{\lambda}})$ is said to satisfy the strong second-order sufficient condition~\cite{Robinson-1980-strongreg} for problem~\eqref{eq:general-form} if there is a $\rho \in \mathcal{R}$ such that 
\begin{equation} \label{eq:ssosc}
\begin{aligned}
 & \langle \vec{\omega}, \nabla_{xx}^2 \mathcal{L}(\bar{\vx}, \bar{\vec{\lambda}})\vec{\omega} \rangle + \rho \sum_{i}\langle \nabla c_i (\bar{\vx}), \vec{\omega} \rangle^2 > 0 \\
 &\quad \forall \vec{\omega} \in T_\mathcal{C} (\bar{\vx})/\{0\} 
 \end{aligned}
\end{equation}
where $T_\mathcal{C} (\vx)$ is the tangent cone of $\mathcal{C}$ at point $\vx$. 

We describe the proximal method of multipliers for the general form of problem~\eqref{eq:general-form} in Algorithm~\ref{algo}.
% algorithm framework
\begin{algorithm}[t] \label{alg-ppm}
\caption{Proximal Method of Multipliers}
\label{algo}
\begin{algorithmic}[1]
\State \text{Input: Choose $\vx_0$, $\vec{\lambda}_0$, set $k = 0$.} 
\State $\text{Repeat}$
\State $ \quad \vx_{k+1}  \coloneqq \underset{\vx \in \mathcal{C} } {\mathrm{arg min}} ~ \mathcal{L}_{\mathcal{A}} (\vx,  \vec{\lambda}_k, \sigma_k)$ \\ $\quad \quad \quad \quad \qquad \qquad + \frac{1}{2\sigma_k} \| \vx - \vx_k\|^2 \quad \quad (P^k)$
\State $ \quad \vec{\lambda}_{k+1} \coloneqq \vec{\lambda}_k + \sigma_k c(\vx_{k+1})$ 
\State $ \quad k \coloneqq k+1$ 
\State \text{Until converged}
\end{algorithmic}
\end{algorithm}

%\textbf{Theorem 1:}  Let ($\bar{\vx}$, $\bar{\vec{\lambda}}$) be a KKT pair for problem~\eqref{eq:general-form} satisfying the strongly second order sufficient condition and assume the gradients $\nabla c(\bar{\vx})$ are linearly independent. If the $\{\sigma_k\}$ are large enough with $\sigma_k \to \bar{\sigma} \leq \infty$ and if $\| (\vx_0, \vec{\lambda}_0) - (\bar{\vx}, \bar{\vec{\lambda}}) \|$ is small enough, then there exists a sequence $\{ (\vx_k, \vec{\lambda}_k)\}$ conforming to Algorithm 1 along with open neighborhoods $\mathcal{C}_k$ such that for each $k$, $\vx_{k+1}$ is the unique solution in $\mathcal{C}_k$ to $(P^k)$. Then also, the sequence $\{ (\vx_k, \vec{\lambda}_k)\}$ converges linearly and Fej\'er monotonically to $\bar{\vx}$, $\bar{\vec{\lambda}}$ with rate $r(\bar{\sigma}) < 1$ that is decreasing in $\bar{\sigma}$ and $r(\bar{\sigma}) \to 0$ as $\bar{\sigma} \to \infty$.

\begin{theorem}\label{theorem:1}
 Let ($\bar{\vx}$, $\bar{\vec{\lambda}}$) be a KKT pair for problem~\eqref{eq:general-form} satisfying the strongly second order sufficient condition and assume the gradients $\nabla c(\bar{\vx})$ are linearly independent. If the $\{\sigma_k\}$ are large enough with $\sigma_k \to \bar{\sigma} \leq \infty$ and if $\| (\vx_0, \vec{\lambda}_0) - (\bar{\vx}, \bar{\vec{\lambda}}) \|$ is small enough, then there exists a sequence $\{ (\vx_k, \vec{\lambda}_k)\}$ conforming to Algorithm~\ref{algo} along with open neighborhoods $\mathcal{C}_k$ such that for each $k$, $\vx_{k+1}$ is the unique solution in $\mathcal{C}_k$ to $(P^k)$. Then also, the sequence $\{ (\vx_k, \vec{\lambda}_k)\}$ converges linearly and Fej\'er monotonically to $\bar{\vx}$, $\bar{\vec{\lambda}}$ with rate $r(\bar{\sigma}) < 1$ that is decreasing in $\bar{\sigma}$ and $r(\bar{\sigma}) \to 0$ as $\bar{\sigma} \to \infty$.
\end{theorem}

\begin{proof} (Note that the theorem and proof are revisions and specializations of Theorem 19 from~\cite{Pennanen-2002-proxMOR}.) By Robinson (1980, Theorem 4.1) \cite{Robinson-1980-strongreg}, the strongly second-order sufficient condition and the linear independence condition imply that the KKT system for~\eqref{eq:general-form} is strongly regular at $(\bar{\vx}, \bar{\vec{\lambda}})$. 

When we solve the subproblem $(P^k)$ with explicit bound constraints, from Proposition~\ref{prop}, we actually solve
\begin{displaymath}
\nabla f(\vx) + \frac{1}{\sigma_k} (\vx - \vx_k) + N_{\mathcal{C}} (\vx) + \nabla c(\vx)^*(\vec{\lambda}_k + \sigma_k c(\vx)) \ni 0.
\end{displaymath}
Thus, Algorithm~\ref{algo} is equivalent to Algorithm 3 in \cite{Pennanen-2002-proxMOR} (their general algorithm), and by Theorem 17 of \cite{Pennanen-2002-proxMOR}, we have the local convergence result stated in Theorem~\ref{theorem:1}. 

It remains to show that for large enough $\sigma_k$,  the unique stationary point is in fact a minimizer of $(P^k)$. We apply the second-order sufficient condition in 13.26 from \cite{Rockafellar-2009-variational} and the analogous derivation in the proof of Theorem 19 of \cite{Pennanen-2002-proxMOR}. Then a sufficient condition for $\vx_{k+1}$ to be a local minimizer of $(P^k)$ is that 
\begin{equation*} 
\begin{aligned}
 & \langle \vec{\omega}, \nabla_{xx}^2 \mathcal{L}(\vx_{k+1}, \vec{\lambda}_{k+1})\vec{\omega} \rangle + \frac{1}{\sigma_k} \| \vec{\omega} \|^2 + \\
& \qquad \sigma_k \sum_{i}\langle \nabla c_i (\vx_{k+1}), \vec{\omega} \rangle^2 > 0, \forall \vec{\omega} \in T_\mathcal{C} (\vx_{k+1})/\{0\}. 
 \end{aligned}
\end{equation*}
This condition holds by the continuity of $\nabla_{xx}^2 \mathcal{L}$ and $\nabla c_i$, and by~\eqref{eq:ssosc}, provided $\sigma_k$ is large enough. \hfill$\blacksquare$
\end{proof}

A main assumption for the analysis above is that we can solve the subproblem $(P^k)$ exactly. This was adjusted in \cite{Iusem-2003-inexact-prox}, which showed local convergence for approximate solutions of $(P^k)$. 

\section{Simplified Alternating Direction Method of Multipliers}
\label{sec:sadmm}
%For instance, softly imposing constraints, or using proximal point operations, such that we do not need to use LBFGSB to solve any subproblems resulted in a slower overall time-to-solution.
One downside to both of the proposed methods is that they involve using the L-BFGS-B method to solve the bound-constrained non-convex objectives in the substeps. This is a complex routine with a high runtime itself. In this section, we are interested in seeing if there are simplified ADMM variants that might futher improve runtime by avoiding this non-convex solver. 
This corresponds to, for example, inexact ADMM (allowing inexact primal alternating minimization solutions, e.g., one proximal gradient step per block).

%Thus, we are interested in testing whether these variants can be even faster than the proposed method described in Section~\ref{sec:admm}.
In the ADMM method from Section~\ref{sec:admm}, we know that updating the block variables $\vf$, $\vg$, $\vs$ and $r$ is simple and convex, so we can get globally optimal solutions. The only hard part is to update $\mY$, which is non-convex. However, there are few results about convergence for ADMM in the non-convex case as well as the case with multiple blocks, i.e., more than two blocks of variables (e.g., $\mY$, $\vf$, $\vg$, $\vs$ and $r$) that would apply to this problem. For instance, in~\cite{admm-variant1}, it has been shown that an ADMM method does not converge for a multi-block case even for a convex problem.
%a negative example is given for ADMM multi-block case even for a convex problem. 

In fact, in our preliminary investigations, many common variations on the ADMM methods did not yield any performance increase or resulted in slower performance or did not converge at all. For example, we tried to avoid the L-BFGS-B in the update for $\mY$ by simply using a step of projected gradient descent instead. We found the resulting Simplified ADMM (SADMM) method converges much slower than our ADMM method with the non-convex solver (more details are in Section~\ref{sec:karate}). The same experiment with multiple steps of projected gradient descent only performed worse.

Therefore, common accelerated variants of ADMM proposed for convex problems with two-block case do not necessarily improve the performance of ADMM in our problem. We believe that the NEO-K-means low-rank SDP problem will be a useful test case for future research in this area.

\section{Experimental Results}
\label{sec:results}
\begin{figure*}[htbp]
\centering
\begin{minipage}[b]{1\linewidth}\centering
  \centering
  \begin{tabular}{cc}
    \subfloat[The infinity norm of the constraints vector vs. Time]{\includegraphics[width=0.45\textwidth]{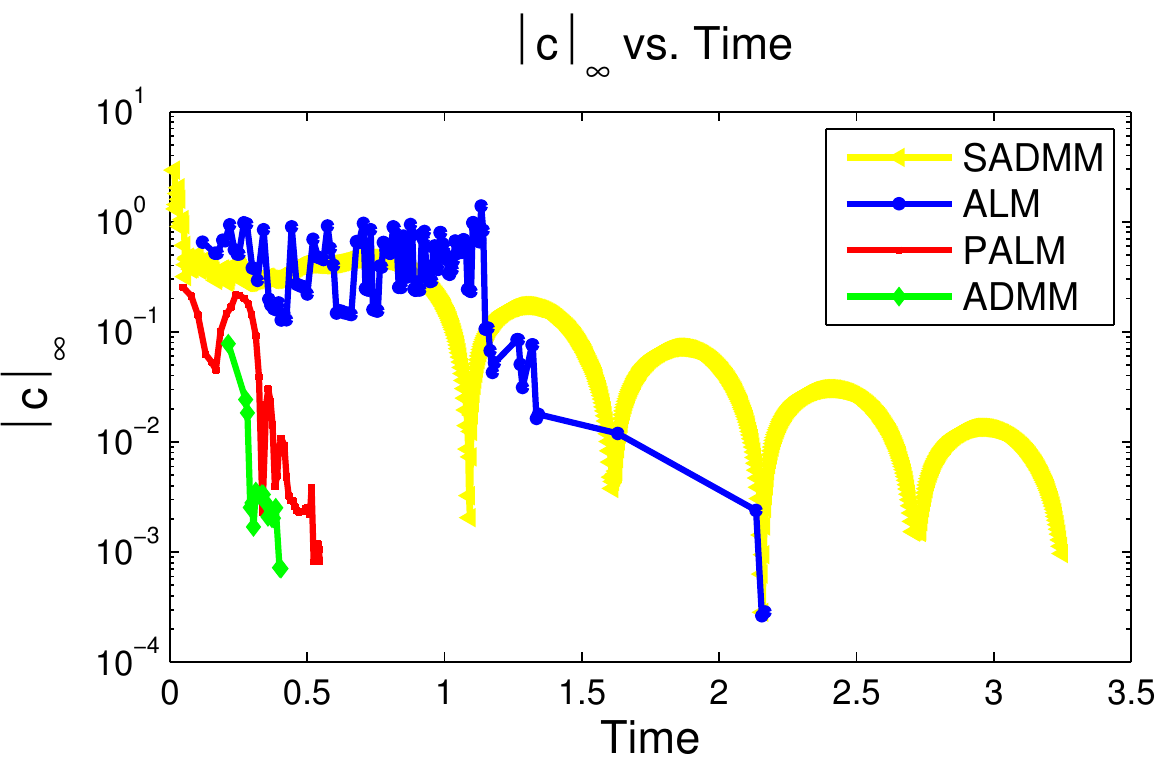}}&
    \subfloat[Objective function values vs. Time]{\includegraphics[width=0.45\textwidth]{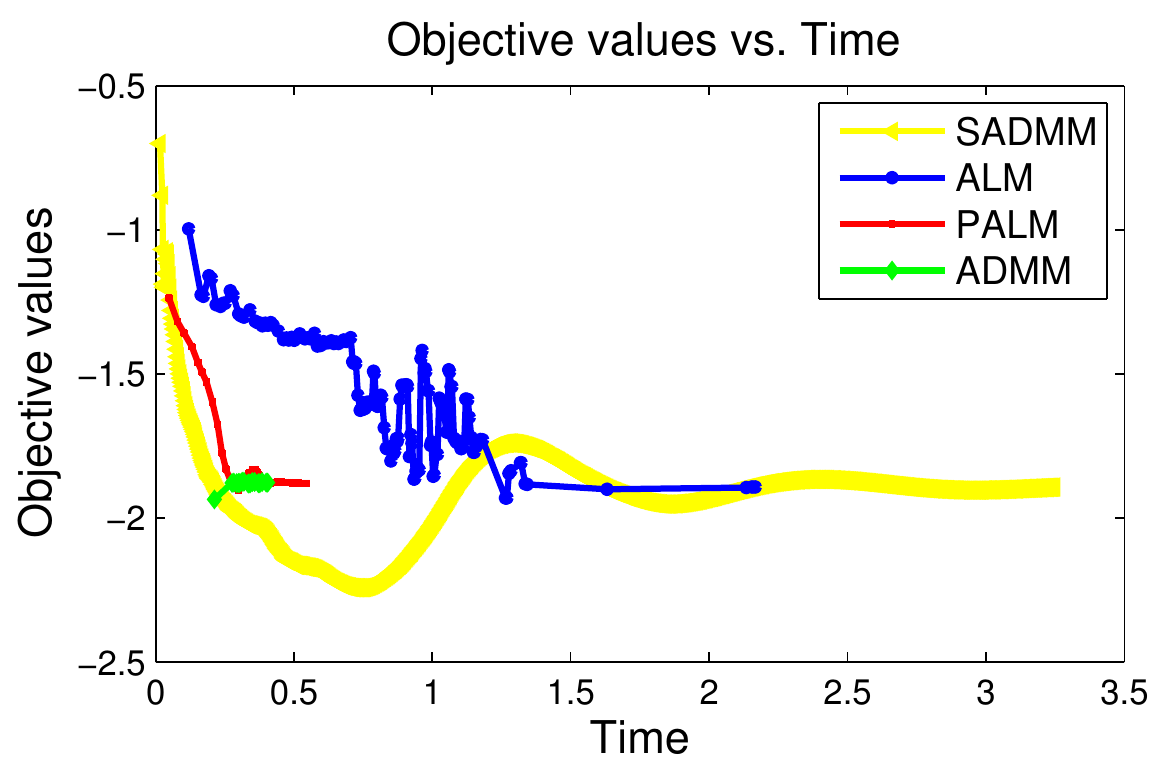}}
  \end{tabular}
\end{minipage}
\caption{The convergence behavior of ALM, PALM, ADMM and SADMM on a Karate Club network. PALM and ADMM converge faster than ALM while SADMM is much slower.}
\label{karate}
\vspace{-0.5cm}
\end{figure*}

In this section, we demonstrate the efficiency of our proposed methods on real-world problems. Our primary goal is to compare our two new methods, PALM and ADMM with the classical augmented Lagrangian method (denoted by ALM) in terms of their ability to optimize~\eqref{eq:neo-lr}. All these three algorithms are implemented in MATLAB and use the L-BFGS-B routine~\cite{Byrd-1995-lbfgsb} written in Fortran to solve the bound-constrained non-linear subproblems. 
%This solve step dominates the runtime of each method.

\subsection{Convergence Analysis on the Karate Club Network.}
\label{sec:karate}
We first illustrate the convergence behavior of each of the methods on an overlapping community detection task in a graph. We use the Zachary's karate club network \cite{karate} which is a small social network among 34 members of a karate club.  
%For more on how these methods work on such objective functions, see~\cite{whang-2015-neo}.

In Figure~\ref{karate}, (a) shows the infinity norm of the constraints vector and (b) shows the NEO-K-Means low-rank SDP objective function values defined in \eqref{eq:neo-lr} as time progresses respectively. We set the infeasibility tolerance to be less than $10^{-3}$. Both of our methods, PALM and ADMM, achieve faster convergence than ALM in terms of both the feasibility of the solution and the objective function value mainly because the subproblems for L-BFGS-B are faster to solve. To demonstrate that the common variants of ADMM do not accelerate the convergence in our problem, we also compare with the simplified alternating direction method of multipliers (Section~\ref{sec:sadmm}, denoted by SADMM). Note that for SADMM, we do not need to use L-BFGS-B to solve the subproblems, instead, we use one single gradient-descent step to have the solution inexactly. It is clear to see that SADMM is much slower than ADMM, and even slower than ALM.

%\begin{figure}[t]
%\begin{center}
%\includegraphics[width=0.42\textwidth]{t-cnorm-eps-converted-to.pdf}
%\includegraphics[width=0.42\textwidth]{t-obj-eps-converted-to.pdf}
%\caption{The convergence behavior of ALM, PALM, ADMM and SADMM on a Karate Club network. PALM and ADMM converge faster than ALM while SADMM is much slower.}
%\label{karate}
%\end{center}
%\end{figure}

\subsection{Data Clustering on Real-world Datasets.}

\begin{figure*}[htbp]
\centering
\begin{minipage}[b]{1\linewidth}\centering
  \centering
  \begin{tabular}{ccc}
    \subfloat[Objective values in~\eqref{eq:neo-lr} on YEAST]{\includegraphics[width=0.32\textwidth]{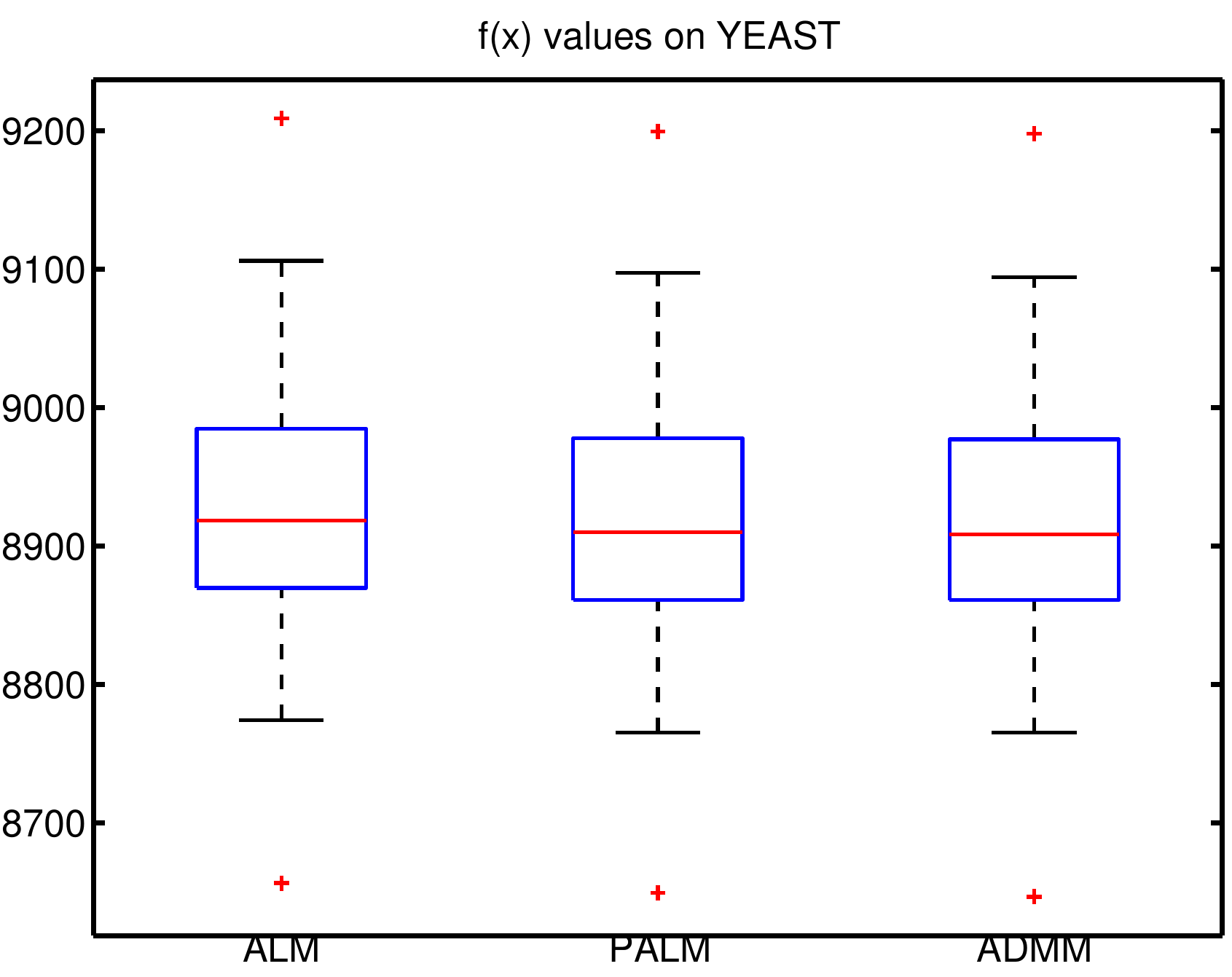}}&
    \subfloat[Objective values in~\eqref{eq:neo-lr} on SCENE]{\includegraphics[width=0.32\textwidth]{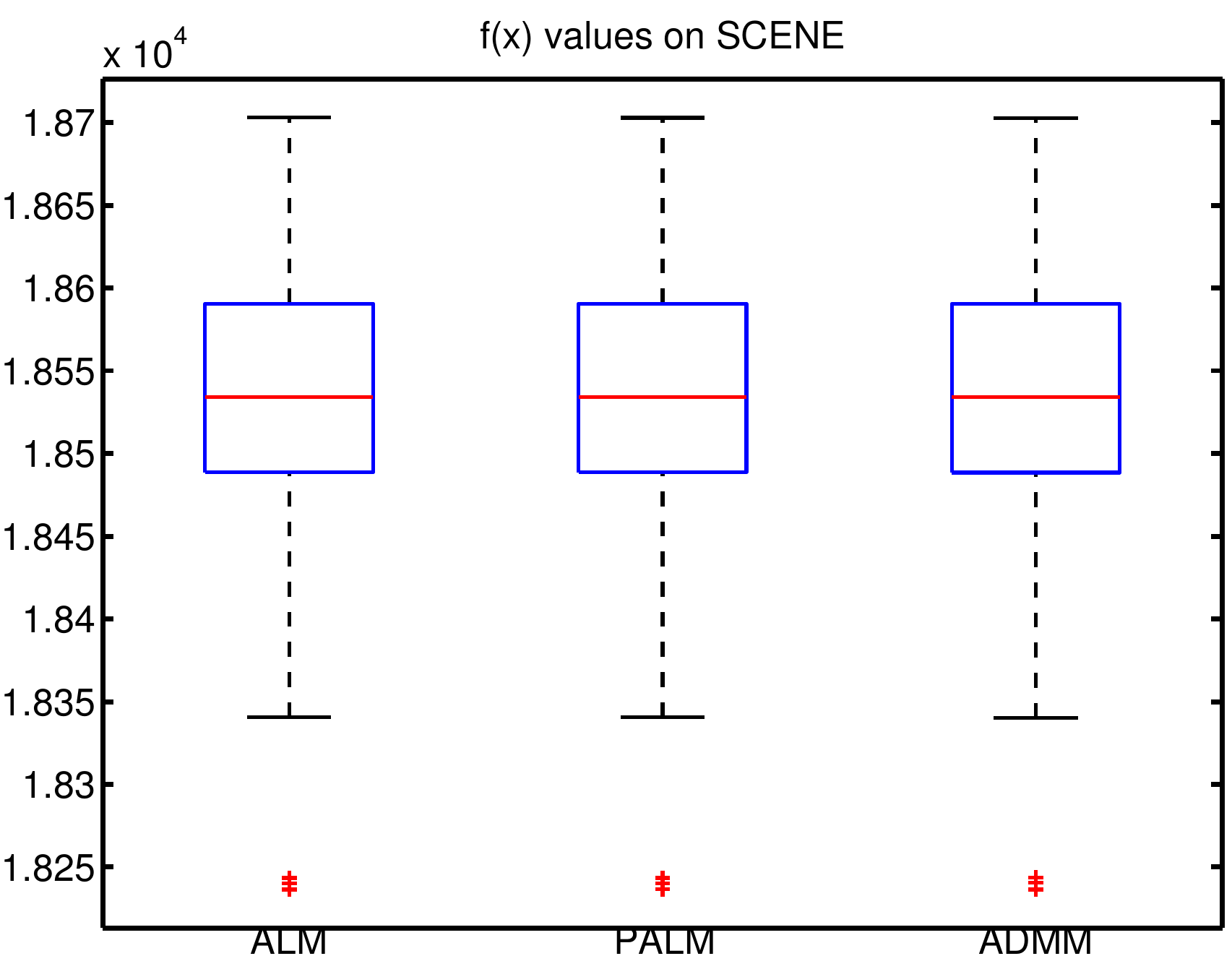}}&
    \subfloat[Objective values in~\eqref{eq:neo-lr} on MUSIC]{\includegraphics[width=0.32\textwidth]{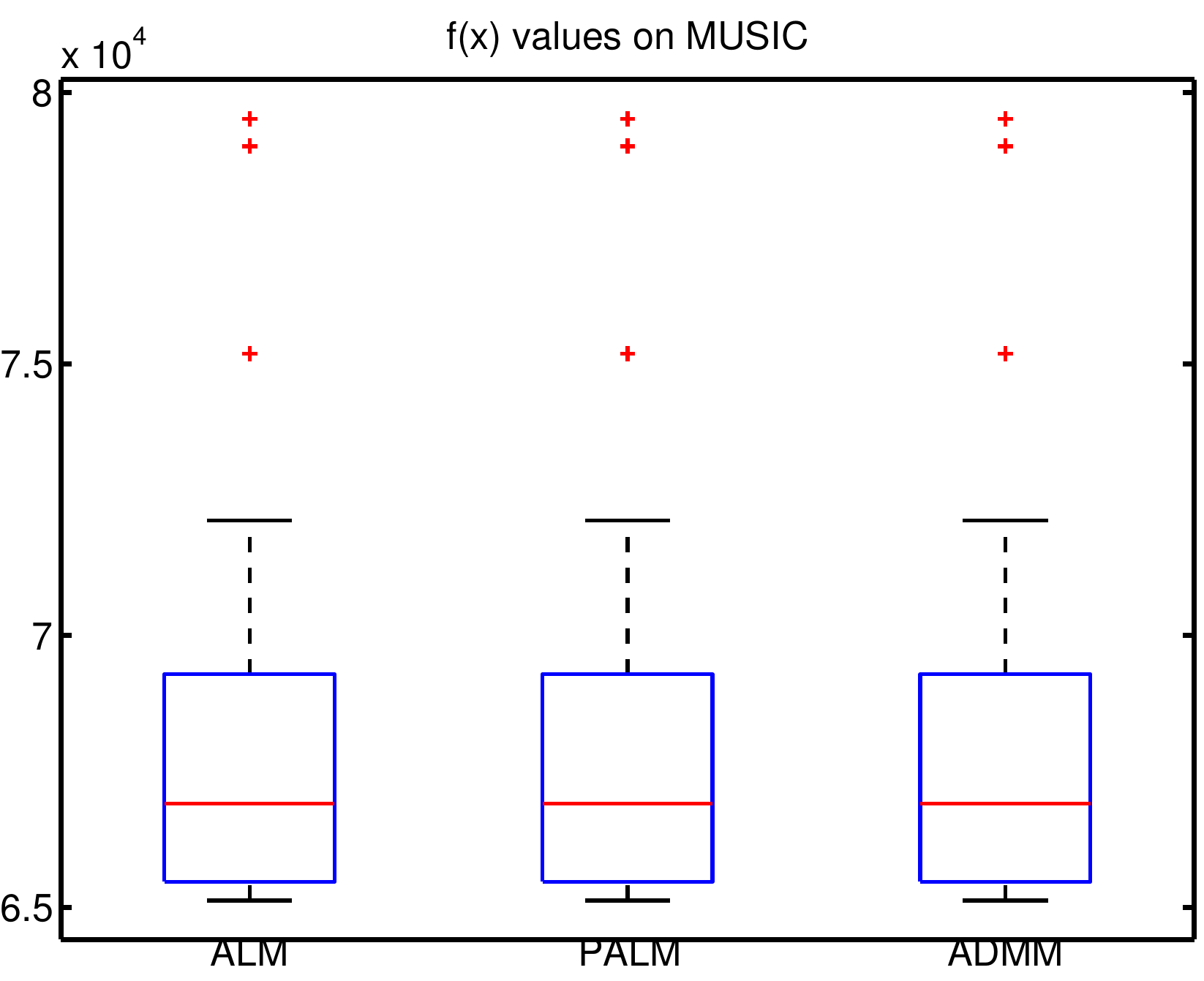}}
  \end{tabular}
\end{minipage}
\centering
\begin{minipage}[b]{1\linewidth}\centering
  \centering
  \begin{tabular}{ccc}
    \subfloat[Runtimes on YEAST]{\includegraphics[width=0.32\textwidth]{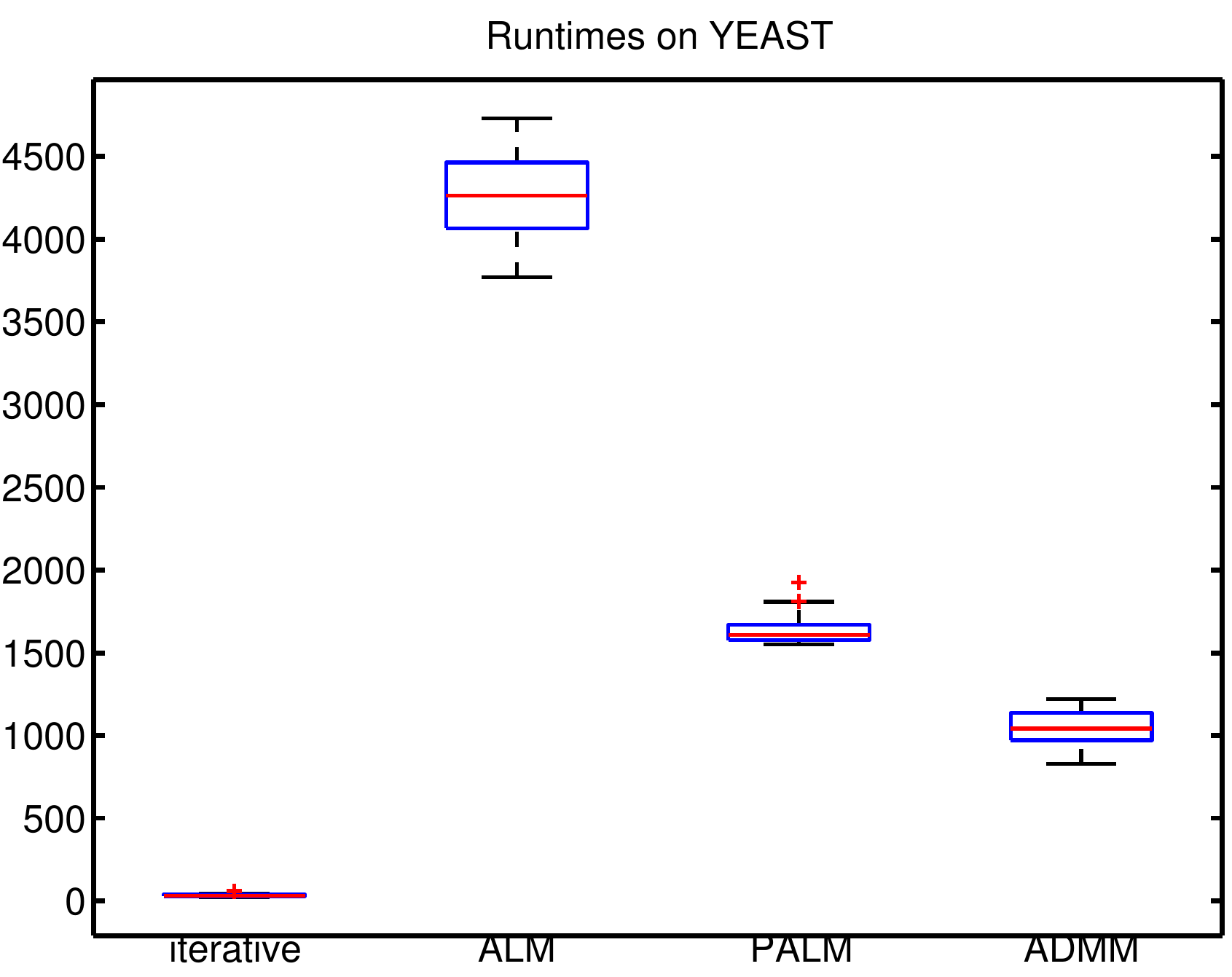}}&
    \subfloat[Runtimes on SCENE]{\includegraphics[width=0.32\textwidth]{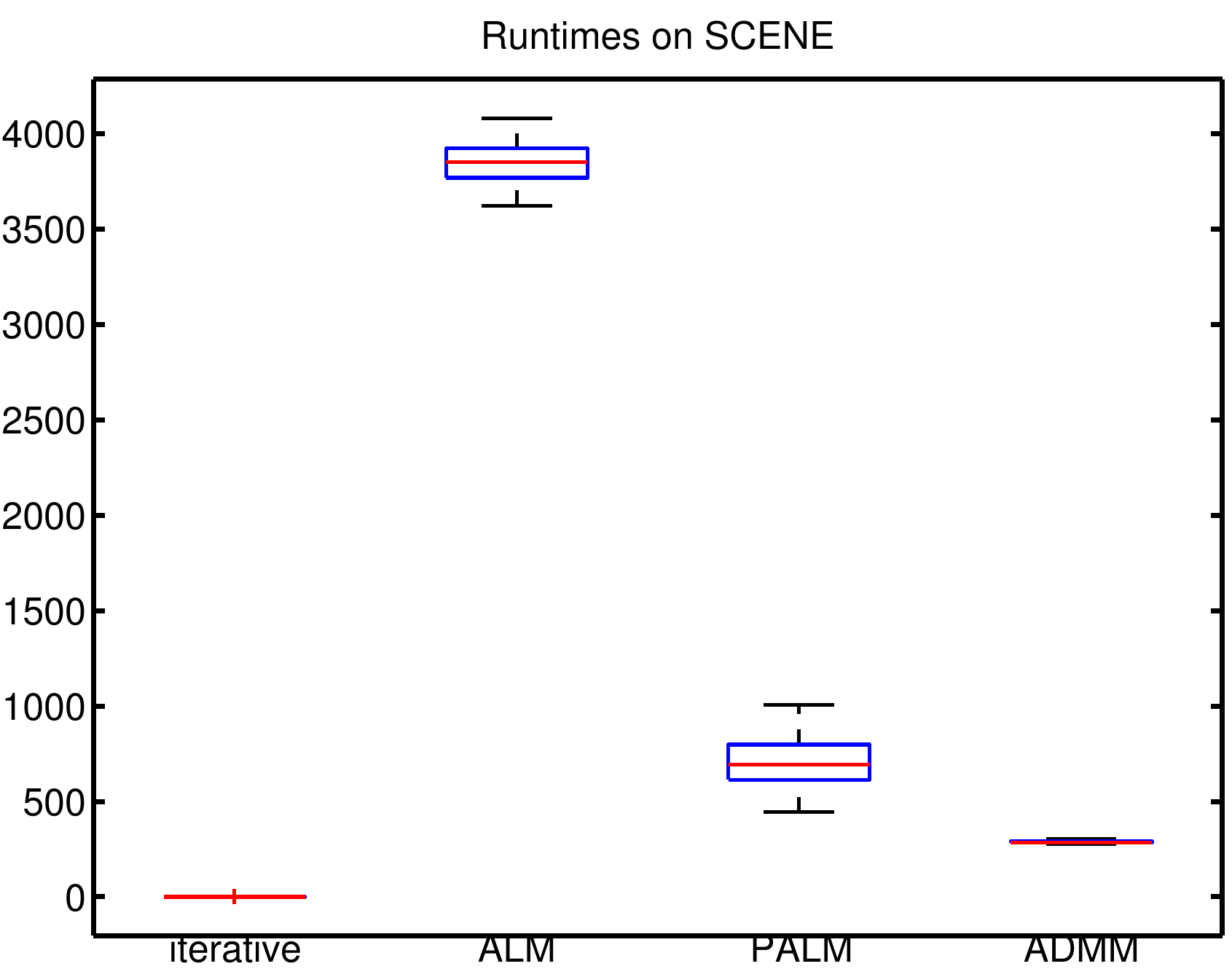}}&
    \subfloat[Runtimes on MUSIC]{\includegraphics[width=0.32\textwidth]{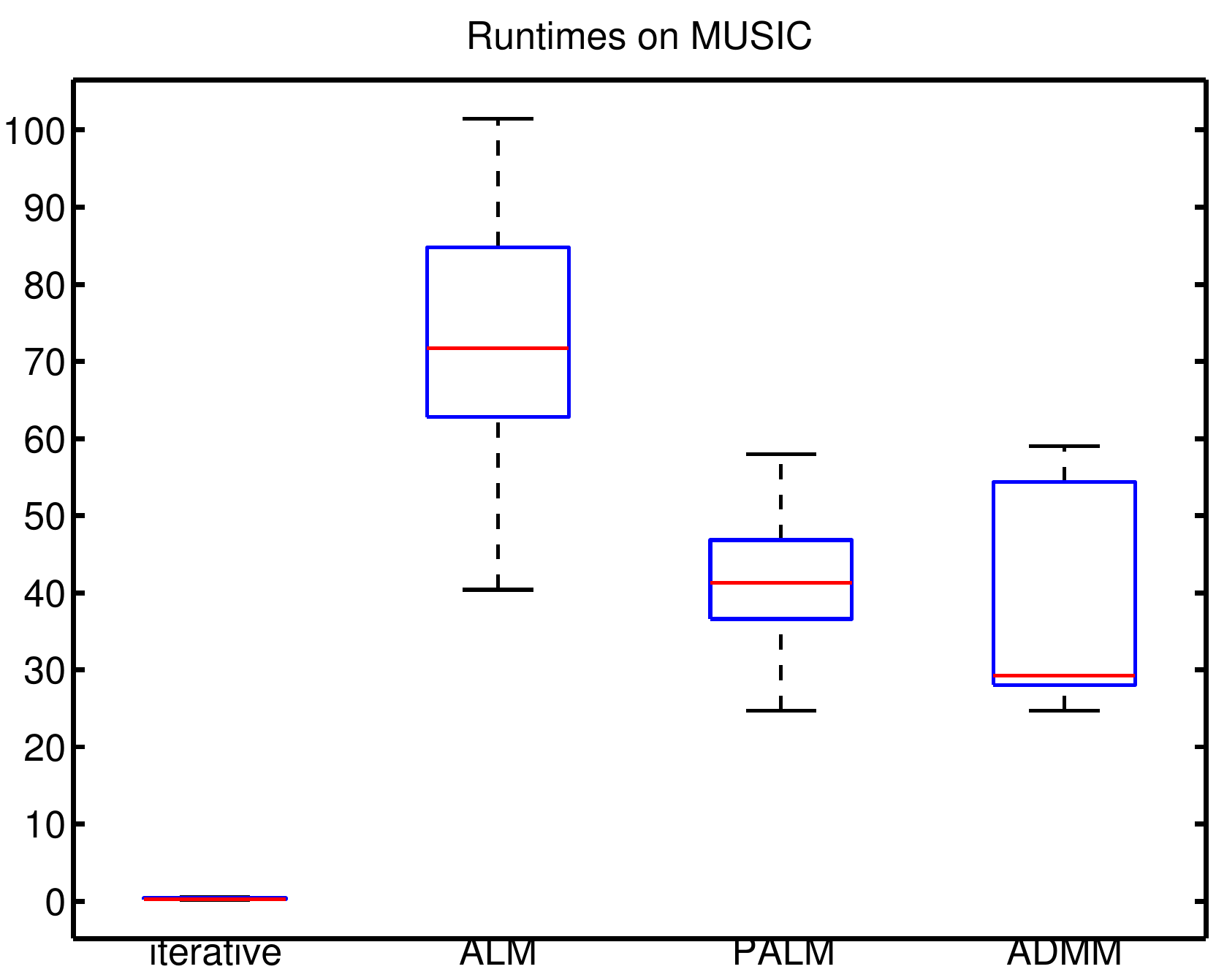}}
  \end{tabular}
\end{minipage}
\centering
\begin{minipage}[b]{1\linewidth}\centering
  \centering
  \begin{tabular}{ccc}
    \subfloat[NEO-K-Means objectives on YEAST]{\includegraphics[width=0.32\textwidth]{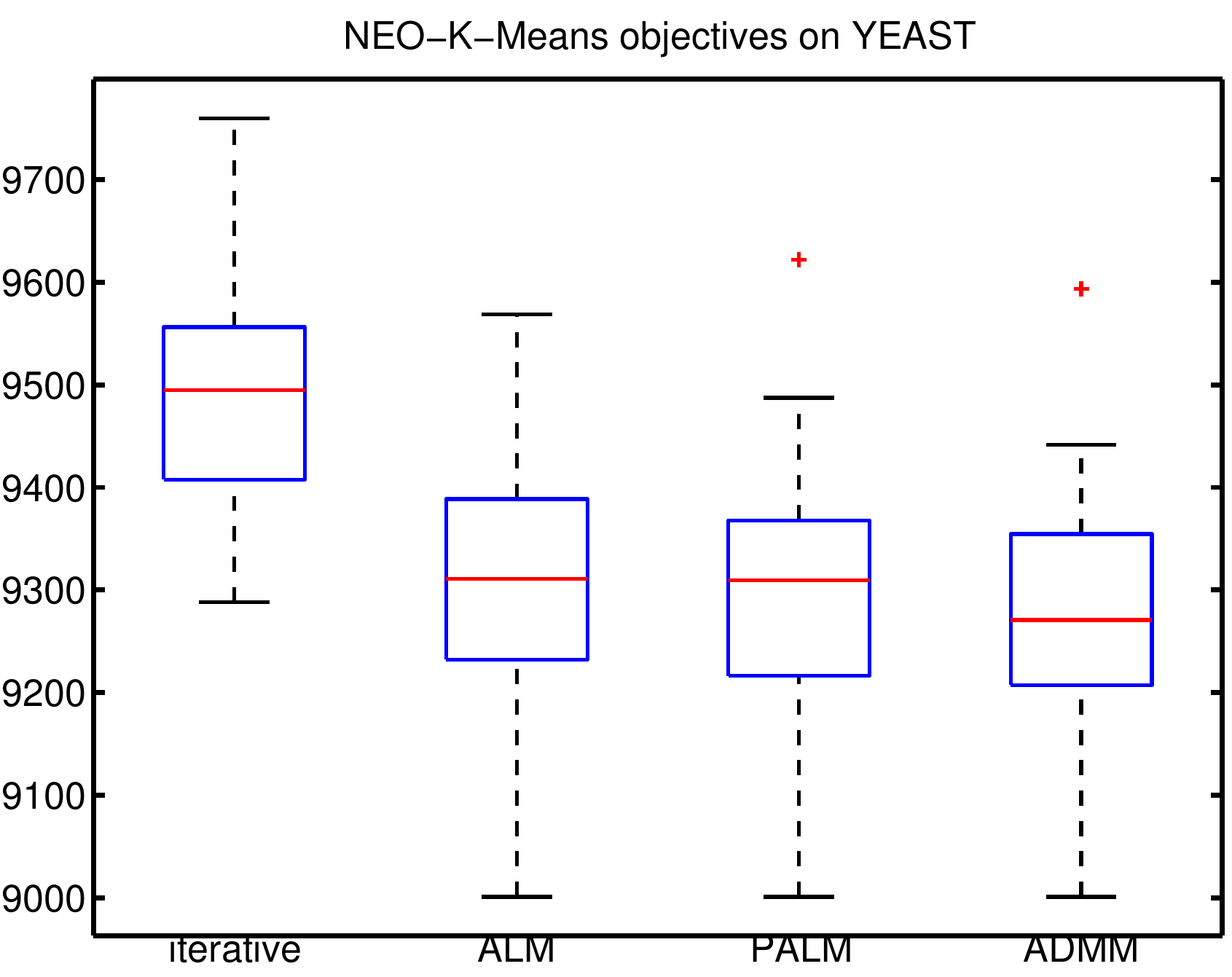}}&
    \subfloat[NEO-K-Means objectives on SCENE]{\includegraphics[width=0.32\textwidth]{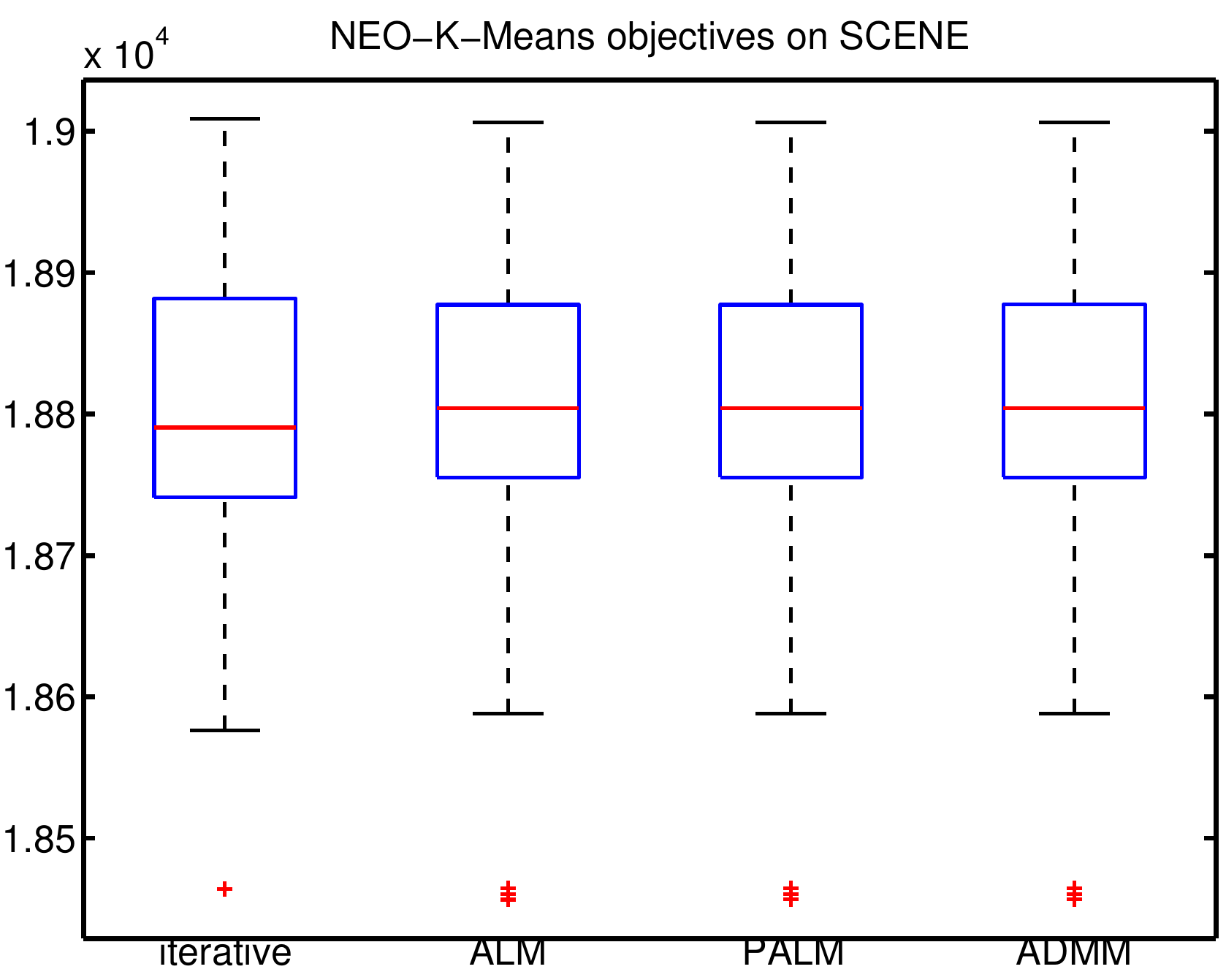}}&
    \subfloat[NEO-K-Means objectives on MUSIC]{\includegraphics[width=0.32\textwidth]{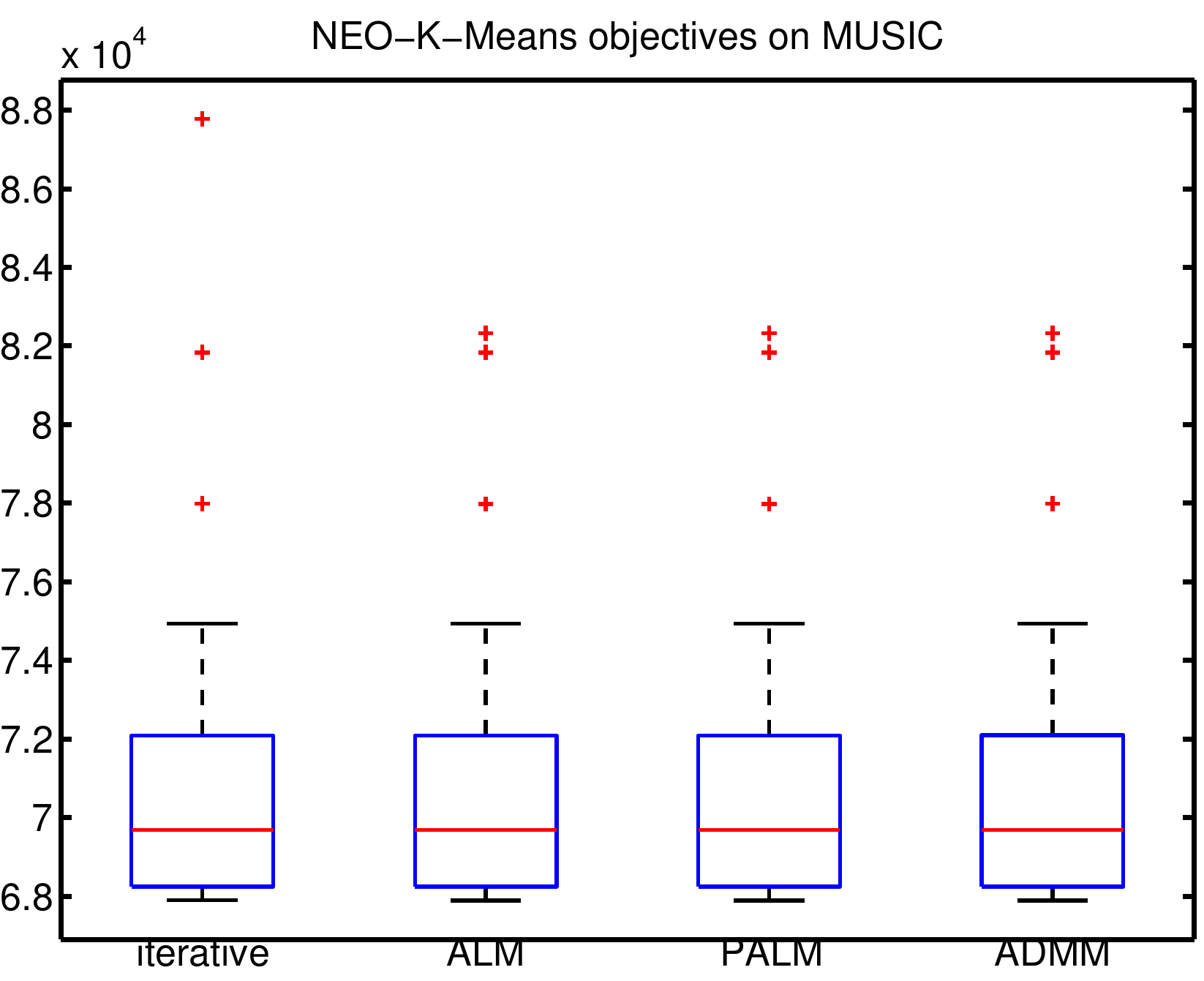}}
  \end{tabular}
\end{minipage}
\centering
\begin{minipage}[b]{1\linewidth}\centering
  \centering
  \begin{tabular}{ccc}
    \subfloat[$F_1$ scores on YEAST]{\includegraphics[width=0.32\textwidth]{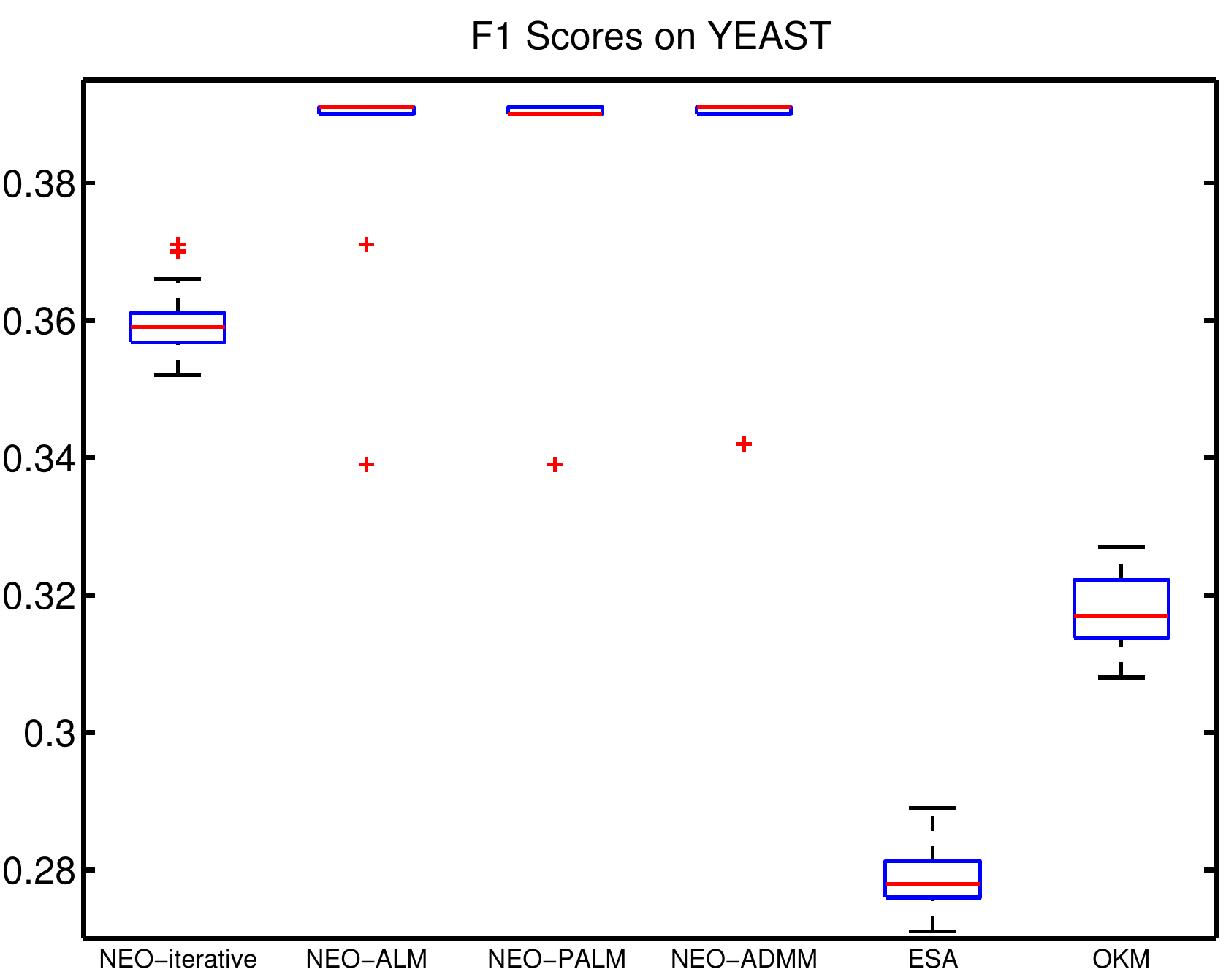}}&
    \subfloat[$F_1$ scores on SCENE]{\includegraphics[width=0.32\textwidth]{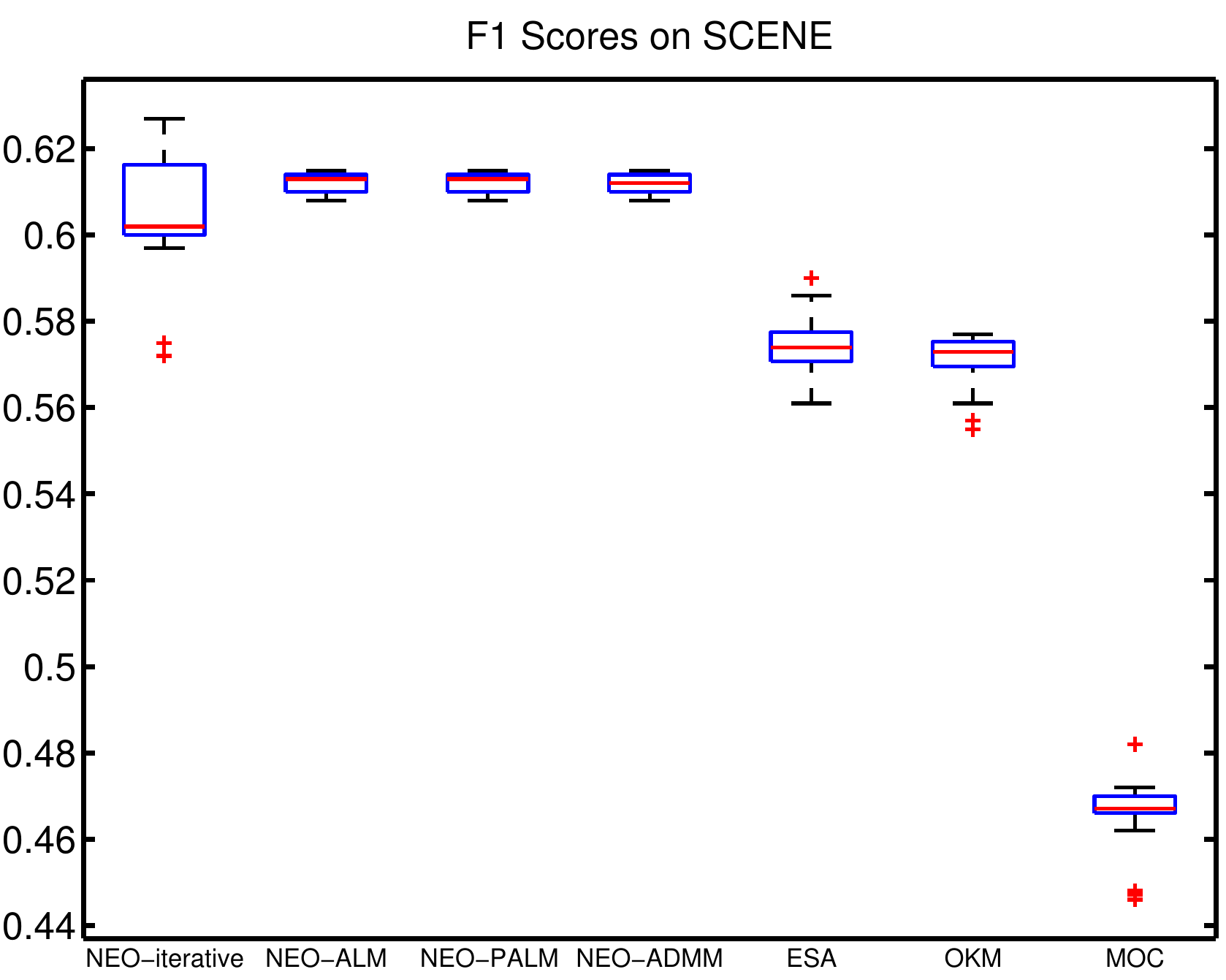}}&
    \subfloat[$F_1$ scores on MUSIC]{\includegraphics[width=0.32\textwidth]{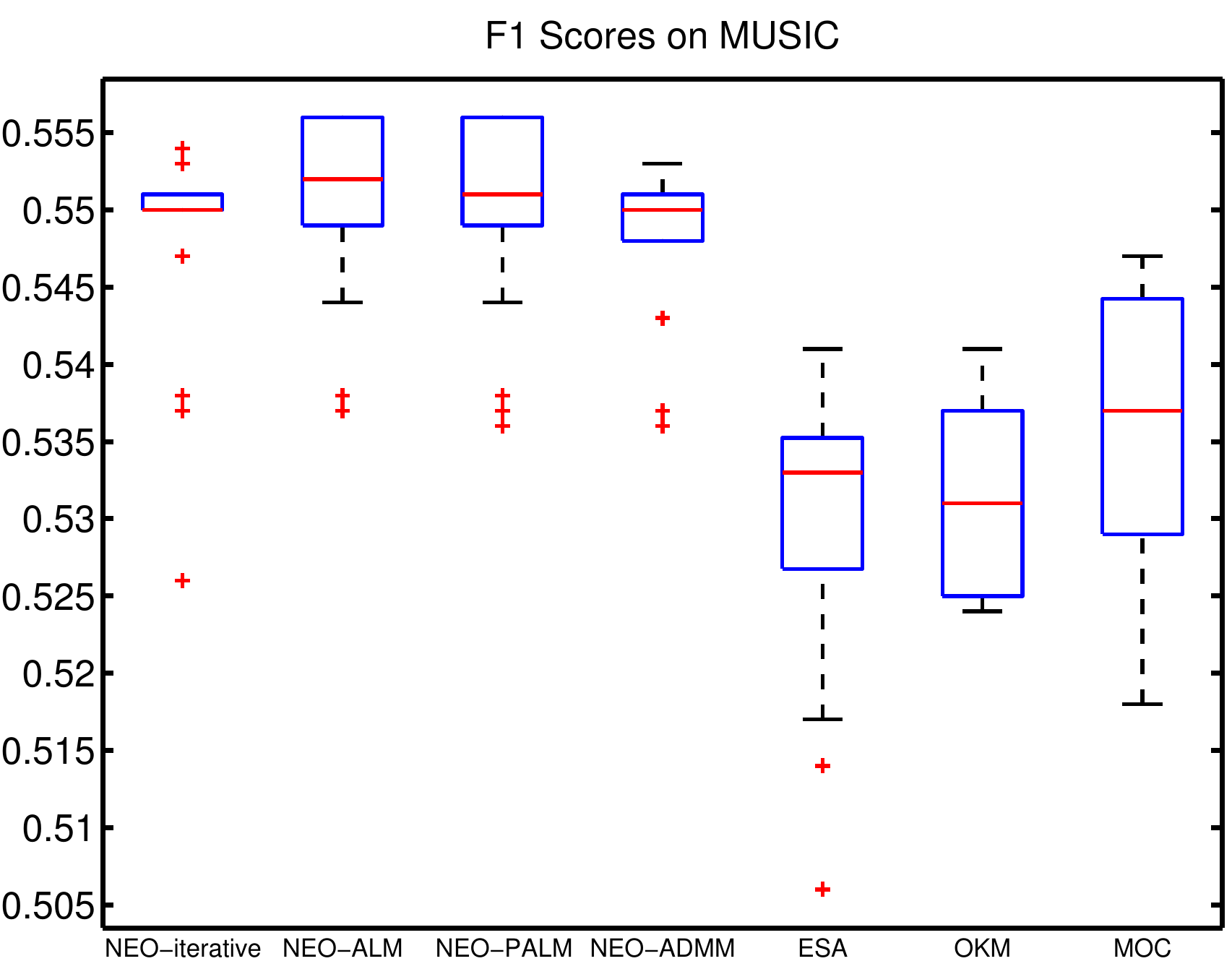}}
  \end{tabular}
\end{minipage}
\caption{Box-plots comparing the results on 25-trials of the algorithms in terms of objective values in~\eqref{eq:neo-lr}, runtimes, NEO-K-Means objective values in~\eqref{eq:neo}, and $F_1$ scores on YEAST, SCENE, and MUSIC datasets. The median performance is indicated by the middle red line and the box shows the 25\% and 75\% percentiles.}
\label{fig_all}
\vspace{-0.5cm}
\end{figure*}

Next, we compare the three methods (ALM, PALM and ADMM) on larger datasets. We use three different datasets from \cite{mulan}. The SCENE dataset~\cite{scene} contains 2,407 scenes represented as feature vectors; the YEAST dataset~\cite{yeast} consists of 2,417 genes where the features are based on micro-array expression data and phylogenetic profiles; the MUSIC dataset~\cite{music} contains a set of 593 different songs. There are known ground-truth clusters on these datasets (we set $k$ as the number of ground-truth clusters; $k$=6 for MUSIC and SCENE, and $k$=14 for YEAST). The goal of this comparison is to demonstrate that PALM and ADMM have performance equivalent to the ALM method, while running substantially faster. We will also compare against the iterative NEO-K-Means algorithm as a reference.

We initialize ALM, PALM, and ADMM using the iterative NEO-K-Means algorithm as also used in~\cite{Hou-Whang-2015-lrsdp-neo}.  The parameters $\alpha$ and $\beta$ in the NEO-K-Means are automatically estimated by the strategies proposed in~\cite{whang-2015-neo}. This initialization renders the method sensitive to the local region selected by the iterative method, but this is usually a high-quality region. We use the procedure from~\cite{Hou-Whang-2015-lrsdp-neo} to round the real-valued solutions to discrete assignments. Briefly, this uses the solution vectors $\vg$ and $\vf$ to determine which points to assign and roughly how many clusters each data point resides in. Assignments are then greedily made based on values of the solution matrix $\mY$.  We run all the methods 25 times on the three datasets, and summarize the results in Figure~\ref{fig_all}. The results from these experiments illustrate the following points:
\begin{compactitem}
	\item (Top-row -- objective values) The PALM, ADMM, and ALM methods are all indistinguishable as far as their ability to optimize the objective of the NEO-K-Means low-rank SDP problem \eqref{eq:neo-lr}.
	\item (Second-row -- runtimes) Both the PALM and ADMM methods are significantly faster than ALM on the larger two datasets, SCENE and YEAST. In particular, ADMM is more than 13 times faster on the SCENE dataset. Since the MUSIC dataset is relatively small, the speedup is also relatively small, but the two new methods, PALM and ADMM are consistently faster than ALM. Note that we do not expect any of the optimization-based methods will be faster than the iterative NEO-K-Means method since it is a completely different type of algorithm (In particular, it optimizes the discretized objective).
\end{compactitem}
Thus, we conclude that the new optimization procedures (PALM and ADMM) are considerably faster than the ALM method while achieving similar objective function values.

The next investigation studies the discrete assignments produced by the methods. Here, we see that (third row of Figure~\ref{fig_all}) there are essentially no differences among any of the optimization methods (ALM, PALM, ADMM) in terms of their objective values after rounding to the discrete solution and evaluating the NEO-K-Means objective. The optimization methods outperform the iterative method on the YEAST dataset by a considerable margin. 

%\begin{figure*}[!t]
%\begin{tabular}{ccc}
%    \subfloat[$F_1$ scores on YEAST]{\includegraphics[width=0.32\textwidth]{./yeast_f1_all_cr-eps-converted-to.pdf}}&
%    \subfloat[$F_1$ scores on SCENE]{\includegraphics[width=0.32\textwidth]{./scene_f1_all_cr-eps-converted-to.pdf}}&
%    \subfloat[$F_1$ scores on MUSIC]{\includegraphics[width=0.32\textwidth]{./music_f1_all_cr-eps-converted-to.pdf}}
%\end{tabular}
%\caption{Box-plots of the $F_1$ scores on 25-trials of the algorithms on YEAST, SCENE, and MUSIC datasets.}
%\label{fig_f1}
%\vspace{-0.5cm}
%\end{figure*}

Finally, to see the clustering performance, we compute the $F_1$ score which measures the matching between the algorithmic solutions and the ground-truth clusters in the last row of Figure~\ref{fig_all}. Higher $F_1$ scores indicate better alignment with the ground-truth clusters. We also compare the results with other state-of-the-art overlapping clustering methods, MOC~\cite{moc}, ESA~\cite{sa}, and OKM~\cite{okm}. On the YEAST dataset, MOC returns 13 empty clusters and one large cluster which contains all the data points. So, we do not report $F_1$ score of MOC on this dataset. We first observe that the NEO-K-Means based methods (denoted by NEO-*) are able to achieve higher $F_1$ scores than other methods. When we compare the performance among the three NEO-K-Means optimization methods (NEO-ALM, NEO-PALM, and NEO-ADMM), there is largely no difference among these methods except for the MUSIC dataset. On the MUSIC problem, the ADMM method has a slightly lower $F_1$ score than PALM or ALM. This is because the objective values obtained by ADMM on this dataset seem to be minutely higher than the other two optimization strategies and this manifests as a noticeable change in the $F_1$ score. In this case, however, the scale of the variation is low and essentially, the results from all the NEO-K-Means based methods are equivalent. On the SCENE dataset, the iterative algorithm (NEO-iterative) can sometimes outperform the optimization methods in terms of $F_1$ although we note that the median performance of the optimization is much better and there is essentially no difference between NEO-PALM, NEO-ADMM, and NEO-ALM. On the YEAST dataset, the reduced objective function value corresponds with an improvement in the $F_1$ scores for notably better results than NEO-iterative.

\section{Discussion}
\label{sec:discuss}
Overall, the result from the previous section indicate that both the PALM and ADMM methods are faster than ALM with essentially no change in quality.  Thus, we can easily recommend them instead of ALM for optimizing these low-rank objectives. There is still a substantial gap between the performance of the simple iterative algorithm and the optimization procedures we propose here. However, the optimization procedures avoid the worst-case behavior of the iterative method and result in more robust and reliable results as illustrated on the YEAST dataset and in other experiments from \cite{Hou-Whang-2015-lrsdp-neo}.

%The methods we propose here fit into a broad category of practical techniques when current SDP solvers and hardware are insufficient for the problem scale~\cite{Burer-2003-SDPLR,Lee-2010-max-norm,Lang2005-spectral-weaknesses}. In general, the strategy for these problems is to use a low-rank decomposition of the matrix variable and suffer the loss-of-convexity that results. Also, our work is related to the general strategy of continuous optimization for clustering~\cite{kulis2007,rachel}. Additional work in this area includes a different type of convex clustering objective~\cite{Hocking-2011-clusterpath,Lindsten-2011-convex-kmeans} that essentially assigns a centroid for each data point. Clustering occurs when these centroids are subjected to a fusion penalty term. Recent work has shown that ADMM methods are also practically useful~\cite{Chi-2015-splitting} for this problem. 

In terms of future opportunities, we are attempting to identify a convergence guarantee for the ADMM method in this non-convex case. This would put the fastest method we have for the optimization on firm theoretical ground. In terms of the clustering problem, we are exploring the integrality properties of the SDP relaxation itself~\cite{Hou-Whang-2015-lrsdp-neo}. Our goal here is to show a result akin to that proved in~\cite{rachel} about integrality in relaxations of the $k$-means objective. Finally, another goal we are pursuing involves understanding when our method can recover the partitions from an overlapping block-model with outliers. This should hopefully show that the optimization approaches have a wider recovery region than the simple iterative methods and provide a theoretical basis for empirically observed improvement.

%\vspace{-0.3cm}
{\small
\subsection*{Acknowledgments}
This research was supported by NSF CAREER award CCF-1149756 to DG, and by NSF grants CCF-1117055 and CCF-1320746 to ID.
%\vspace{-0.4cm}
}

\bibliographystyle{abbrv}
\bibliography{lrsdp_sdm} 

\end{document}